\documentclass[aps,pra,onecolumn, superscriptaddress, a4paper, longbibliography]{revtex4-2}

\usepackage[utf8]{inputenc} 
\usepackage{url}            
\usepackage{amsfonts}       
\usepackage{nicefrac}       
\usepackage{epsfig}
\usepackage{graphicx}
\usepackage{amsmath}
\usepackage{amssymb}
\usepackage{amsthm}

\usepackage{subcaption}
\usepackage{stmaryrd}
\usepackage[noend]{algorithmic}
\usepackage{algorithm}

\floatname{algorithm}{Algorithm}

\newtheorem{lemma}{Lemma}

\graphicspath{ {images/} }

\begin{document}

\title{Estimation of Convex Polytopes for Automatic Discovery of Charge State Transitions in Quantum Dot Arrays}
\author{Oswin Krause}
\affiliation{Department of Computer Science, University of Copenhagen, 2100 Copenhagen, Denmark}

\author{Bertram Brovang}
\affiliation{Center for Quantum Devices, University of Copenhagen, Universitetsparken 5, Copenhagen 2100, Denmark}

\author{Torbj\o rn Rasmussen}
\affiliation{Center for Quantum Devices, University of Copenhagen, Universitetsparken 5, Copenhagen 2100, Denmark}

\author{Anasua Chatterjee}
\affiliation{Center for Quantum Devices, University of Copenhagen, Universitetsparken 5, Copenhagen 2100, Denmark}
\author{Ferdinand Kuemmeth}
\affiliation{Center for Quantum Devices, University of Copenhagen, Universitetsparken 5, Copenhagen 2100, Denmark}

\date{\today}

\begin{abstract}
In spin based quantum dot arrays, material or fabrication imprecisions affect the behaviour of the device, which must be taken into account when controlling it. This requires measuring the shape of specific convex polytopes. In this work, we present an algorithm that automatically discovers count, shape and size of the facets of a convex polytope from measurements. Results on simulated devices as well as a real 2x2 spin qubit array show that we can reliably find the facets of the convex polytopes, including small facets with sizes on the order of the measurement precision.

\end{abstract}
\maketitle 

\section{Introduction}
Quantum devices controlled by gate voltages have wide-ranging applications, spanning from quantum computation, spintronics as well as the exploration of fundamental physics~\citep{nazarov2009quantum}. An important class are spin based quantum dot arrays, which are candidates for universal quantum computers~\citep{loss1998quantum}. In these devices, electrostatic forces are used to trap single electrons at discrete locations, so called quantum dots. Finding the correct control parameters to confine the desired amounts of electrons (e.g., one electron on each dot) or to move single electrons between different quantum dot locations is a key challenge and primary bottleneck in developing these devices.

In this work, we tackle the problem of controlling the electron transitions between quantum dot locations. We present an algorithm that is capable of discovering the set of possible electron transitions as well as their correct control parameters and demonstrate its performance on a real device. Our algorithm is based on a connection to computational geometry and phrases the optimization problem as estimating the facets of a convex polytope from measurements. While this problem is NP-hard, an approximated solution performs well and thus our algorithm represents the first \emph{practical} automatic tuning algorithm which has the prospects to be scaled to more than 2 or 3 quantum dots on real devices. We demonstrate its practicality on a simulated device with four quantum dots as well as a real device with four quantum dots (three qubit dots plus one sensor dot), which is already outside the scope of human tuning capabilities.
%
Our contributions are the following:
\begin{enumerate}
    \item We develop an algorithm that aims to find a sparse approximate solution of a convex polytope from measurements with as few facets as possible. For this we extend the non-convex large margin approach introduced by \cite{kantchelian2014large}.
    \item We proof a lemma that can be used to test the correctness of a learned polytope from measurements. Our result leads to an active learning scheme that iteratively improves the polytope estimate by adding informative samples that systematically disproves the previous solution. 
    \item We show applicability of our algorithm on a real quantum dot array, specifically a foundry-fabricated silicon device that is currently being developed for spin qubit applications. 
\end{enumerate}

\begin{figure}
	\includegraphics[width=\textwidth]{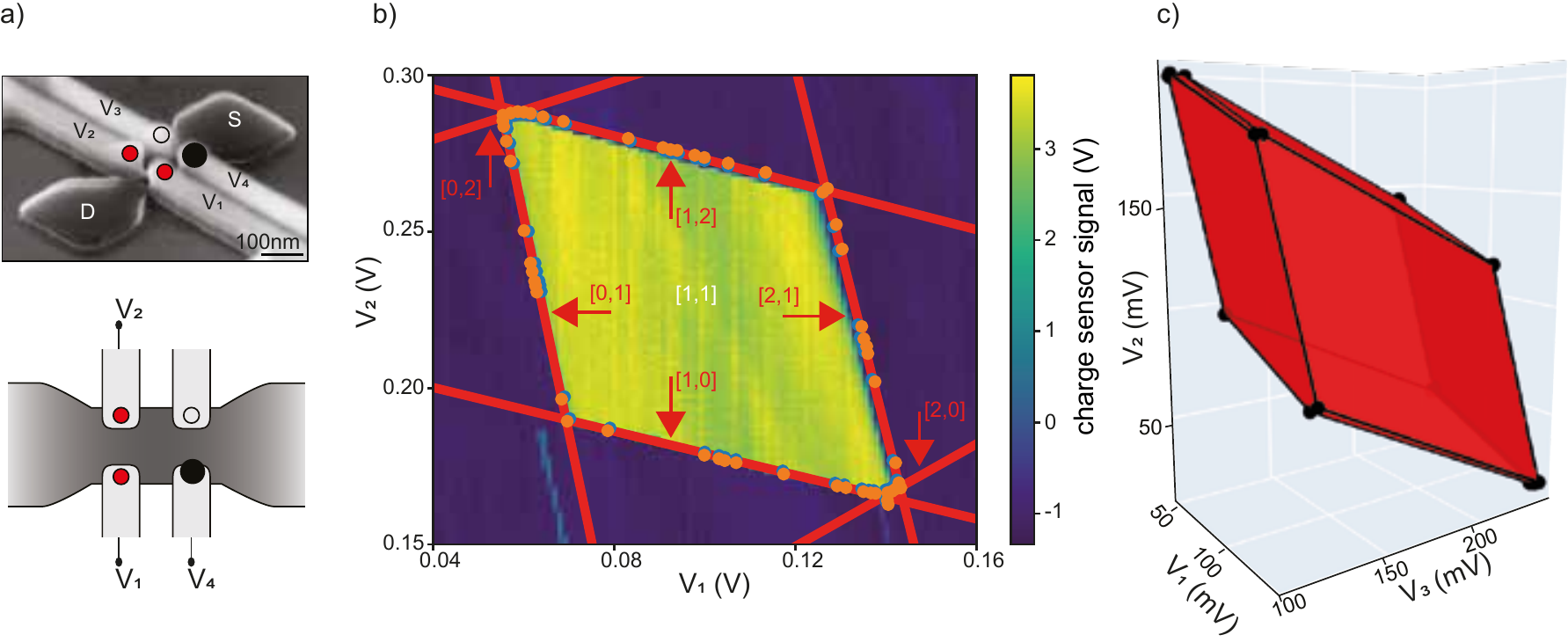}
	\caption{Automatic discovery algorithm implemented on a real quantum dot array with a sensor dot coupled to either a double dot or a triple dot. (a) Micrograph of the device. Individual electrons originating from the source (S) or drain (D) reservoirs can be trapped by gate-induced quantum dots (circles), located below electrostatic gate electrodes G$_{1,\dots4}$ and controlled by voltages $V_{1,\dots4}$. For the targeted $[1,1]$ state, each qubit dot (red) contains one electron, whereas the sensor dot (black) is used to generate a high bandwidth sensor signal. (b) 2D map of the sensor signal as a function of the control voltages $V_1$ and $V_2$, acquired in high resolution to illustrate the convex polytope of this device configuration. For all yellow pixels, the qubit array is in the $[1,1]$ state. (Whenever $V_{1,\dots3}$ is changed, $V_4$ is not held constant but compensates the sensor dot potential for capacitive cross talk from G$_1$, G$_2$ and G$_3$, see experiments.) Our algorithm estimates state transitions to other states (red lines), based on a small number of point pairs $(x^+, x^-)$ (blue and orange dots) obtained via line searches. Target states of transitions are annotated by $[x,y]$. (c) Estimated state boundaries of the 3D polytope associated with single-electron occupation of three quantum dots (under G$_1$, G$_2$ and G$_3$). The ground truth of this triple-dot system is already tedious to measure and unknown for this device tuning. 
	\label{fig:device_experiment}}
\end{figure}

In this work, we target quantum dot arrays such as the one depicted in Figure~\ref{fig:device_experiment}(a) (see also \citep{ansaloni2020single}). In this device, each gate-electrode is able to create a quantum dot below it, by choosing appropriate gate-voltages. We will refer to the vector of electron occupations on each dot as the state of the device. The ground state of the  device is essentially determined by classical physics  and can be described well via the constant interaction model \citep{schroer2007electrostatically}. In this model, the set of control parameters associated with a specific ground state forms a convex polytope. Its boundary is formed by the intersection of linear boundaries, each representing a state transition to another state~\cite{nazarov2009quantum}. With this polytope, the control task can be solved directly by selecting a control parameter path that leaves the current state through the desired state transition into the neighbouring target state. Unfortunately, due to manufacturing imprecisions, the parameters of the constant interaction model are unknown and nonlinear effects lead to small deviations of the model, even though ground-states typically remain convex polytopes. Thus, each manufactured device must be tuned individually in order to find the correct control parameters.

The problem of finding the control parameters becomes daunting in larger devices, as each added dot requires at least one additional gate electrode with its own control parameter. Qubit-qubit connectivities required for quantum simulations and quantum computing also constraints the geometrical layout of the quantum dots and their gate electrodes on a chip: Instead of individual isolated qubit dots with dedicated sensor dots, the current trend is to fabricate dense arrays of coupled qubit dots that are monitored by as few proximal sensor dots as possible. This complicates the tune-up in several ways. First, individual gate voltages no longer act locally, meaning that changing the gate voltage of one qubit also affects the potential of other, nearby qubits. Second, the signal of each sensor dot responds to charge changes of multiple dots, making it more difficult to determine which of the dots is making a transition. For example, in our device, a single charge sensor is responsible for monitoring all quantum dots, which means that its signal must be interpreted and does not directly relate to individual electron transitions~\citep{chanrion2020charge}.

Currently, experts find the correct control parameters manually using ad-hoc tuning protocols~\citep{botzem2018tuning,van2018automated}, which are tedious and time consuming. This limits devices to require tuning of at most three control parameters at once. The field has made first steps into automating parts of this process. For the problem of automating \emph{state identification}, i.e., estimating the electron count on each quantum dot given a set of control parameters, estimation algorithms can be grouped in two main directions. The first direction uses CNN-based approaches~\citep{kalantre2019machine,teske2019machine}, which requires dense rastering of image planes within the parameter space. An alternative approach explores the use of line searches to obtain a dataset of state transitions which are modeled using deep neural networks \citep{zwolak2020ray,zwolak2021ray,PhysRevApplied.13.034075}, which allows human experts to label the discovered regions. Both approaches have only been applied to devices with at most two quantum dots and a maximum of three tunable parameters. 

While promising, these approaches are not easily transferred to the task of finding \emph{state transitions}, as the important transitions are routinely very small and can only be found by high-resolution measurements. This makes rastering of images inefficient and random line searches unlikely to pass through them. Furthermore, modeling state transitions with a deep neural network suffers from the lack of interpretability of these models. Even if the correct model is learned, it is difficult to extract the number and sizes of state transitions, without difficult analysis of the learned model.

We will extend the second approach and use line searches to obtain a dataset, which we iteratively improve upon. We combine this with an interpretable model of the polytope which learns individual state transitions.
In line with the device limitations and measurement challenges outlined above, we will only assume that the sensor reliably detects that a state transition occurred, but not which state the device transitioned to. Using such a sensor, it is possible to locate a state transition in control parameter space using a line search procedure starting from a point within the convex polytope of a state. The resulting line search brackets the position of a state transition by pairs of control parameters $(x^+, x^-)$, each defining one point inside and outside of the (unknown) polytope. This allows for high precision measurements and locates the boundary within a margin of $\delta=\lVert x^+- x^- \rVert$, which is a tuneable parameter chosen based on the trade-off between line search precision and short measurement times.

Under the conditions outlined above, finding the true polytope from measurements is a hard machine learning task. Even if all parameters of the constant interaction model are known, computing the boundary of the polytope requires solving the subspace intersection problem \citep{barber1996quickhull}, with the number of intersections exponential in the number of quantum dots. Indeed, already finding the device state given a set of \emph{known} parameters of the constant interaction model requires solving an integer quadratic program, which is NP-hard in the number of quantum dots \citep{Chaovalitwongse2009}.
Thus, it is unlikely that learning the polytope from measurements is substantially easier. In fact, if we are given a training set of points which are labeled according to whether they are in- or outside the polytope, then finding the optimal large margin polytope is NP-hard. (See \citep{kantchelian2014large} for a review on recent results).
When relaxing the task to allowing a polytope with more facets than the optimal polytope, it has been shown that there exists a polynomial time algorithm  which fits $\ell$ points to a polytope that has at most a factor  of $O(\log \ell)$ more facets than the optimal polytope and has strong approximation guarantees \citep{kantchelian2014large}. However, the algorithmic complexity in our setting is $\ell^{O(1/\delta^2)}$ and thus too large for practical application. It is clear that allowing even more facets will make the task substantially easier, e.g., by adding one facet for each point outside the polytope, the task can be solved using $O(\ell)$ classification tasks, each separating one point outside the polytope from all points inside it. Practical algorithms therefore forego finding the optimal number of facets and use an approximate polynomial time solver  \citep{gottlieb2018learning, gardner2006convergence}. Howevever, there is a lack of numerically stable general purpose optimization algorithms that aim for a \emph{minimal} amount of facets and which go beyond convergence in a volume metric (for example Hausdorff in \citep{ gardner2006convergence}). This is especially problematic in our task, as each facet of the unknown polytope constitutes an operational resource (relocation of individual electrons in the array). Producing an estimate with too many (unphysical) facets might make it impractical to filter out the relevant facets. Our approach is similar to \cite{kantchelian2014large} in that we repeatedly solve convex relaxations of a large margin classification problem. In contrast to the related work, we use a numerically stable second order solver instead of SGD and design the problem such that it produces solutions with as few facets as possible. Random perturbations to the solutions found allow the discovery of different local optima from which we choose the best.

Our machine learning task extends over pure polytope estimation from measurements as our goal is not only to find the optimal polytope given a fixed set of points, but a polytope that generalizes well and accurately reflects the true polytope with all its facets. This is difficult to achieve using an i.i.d. dataset as facets with small surface area are unlikely to be found using random samples. Instead, a scheme is required that actively improves our dataset and ensures that we find all relevant facets that are detectable by our hardware.
We base our scheme on a theoretical result which can be used to assess whether two polytopes are the same. Intuitively, we must perform validation measurements on each facet to show that it is a facet of the true polytope, but we must also search for small (still undetected) facets that might be hiding near vertices of the current polytope. The resulting algorithm performs a line search in the directions of the corners and through each facet of our current best estimate of the polytope. If our polytope is correct, all state transitions of our model will be contained between the pairs $(x^+, x^-)$ of the performed line searches . Otherwise, we obtain examples of new points which we can add to our dataset and fit a new convex polytope. This process is repeated until we either run out of computation time or all performed line searches are correct. The details of our algorithm are given in the method section.

To investigate the performance and quality of our algorithm, we test it on three different setups: Two simulated problems in 3 or 4 dimensions based on either a device simulation or using constructed polytopes from Voronoi regions, and as third application the real device shown in Figure~\ref{fig:device_experiment}(a), where we activate two qubit dots (and the sensor dot) and let human experts verify the result. We further apply the algorithm to the same device with all three qubit dots activated. However, in this case no ground truth is available for comparison. In all simulated experiments, we compare our algorithm to an idealized baseline algorithm that additionally has access to the exact state information while estimating the polytope from measurements and otherwise uses the same active learning protocol. This algorithm assumes a more powerful device with a sensor signal that provides exact state information. (In principle, such a device can be realized by careful design of the charge sensor.) This idealized algorithm is useful to differentiate between the impact of our active learning scheme and the impact of approximately solving the NP-hard estimation problem.

The structure of the paper is as follows: We begin by defining the notation of our paper. In Section~\ref{sec:const_int}, we introduce the constant interaction model which links between polytope estimation and charge transitions in quantum dot devices. In Section~\ref{sec:alg}, we describe the main algorithm by first introducing a meta algorithm for active learning followed by a description of our algorithm for fitting a polytope to measurements. We describe our experiments in Section~\ref{sec:experiments}, present our results in Section~\ref{sec:results} and end with a discussion and conclusion of the paper in Section~\ref{sec:conc}.

\section{Notation}
A convex polytope $P \subset \mathbb{R}^d$ is defined as $P=\{x \in \mathbb{R}^d \mid A_k^Tx+b_k \leq 0, k=1,\dots, N \}$, where $A \in \mathbb{R}^{N \times d}$, $b \in \mathbb{R}^N$, $N$ is the number of inequality constraints and $A_k$ is a column vector representing the $k$-th row of $A$. We will denote the boundary of $P$ as $\partial P=\{x \in P \mid \exists k: A_k^Tx+b_k = 0 \}$. A facet of $P$ is a set of points $f_l=\{x \in P \mid A_l^Tx+b_l=0\} \subset \partial P$. We will also refer to $A_l$ as the normal of $f_l$.  We define as the inside of $f_l$ the open set of points $\{x \in f_l\mid x \notin f_k, k \neq l \}$.
We will call the extreme points of $P$ its vertices. We assume $P$ to be finite and non-degenerate and thus there exist exactly $N$ unique facets.

\section{The Constant Interaction Model}\label{sec:const_int}

In the following, we give a short introduction to the constant interaction model of quantum dot devices \citep{schroer2007electrostatically}. A quantum dot is modeled as a capacitor which can be charged with single electrons. The state of the device $s \in \mathbb{N}^{n_D}$ is the vector of electron occupancies, where, $s_i$ is the number of electrons on the $i$th quantum dot. When applying a vector of gate voltages $V_g \in \mathbb{R}^{n_g}$, the system will assume the state that minimizes the free energy:

\begin{equation}\label{eq:free_energy}
    F(s, V_g) = \frac 1 2 (\lvert e\rvert s + C_{Dg} V_g)^T C_{DD}^{-1} (\lvert e\rvert s +C_{Dg} V_g)
\end{equation}
Here, $\lvert e\rvert$ is the charge of a single electron and the matrices $C_{DD} \in \mathbb{R}^{n_D \times n_D}$ (symmetric positive definite) and $C_{Dg} \in \mathbb{R}^{n_D \times n_g}$ are capacitance matrices which model the interaction strength between dots and between dots and gates, respectively. The state that achieves the minimum free energy given a fixed $V_g$ can be computed by iterating over all possible states:

\begin{equation*}
    s^*(V_g)= \arg \min_{s \in \mathbb{N}^d } F(s,V_g)
    = \arg \min_{s \in \mathbb{N}^d } \frac {\lvert e\rvert^2} 2  s^T C_{DD}^{-1} s 
    + \lvert e\rvert s^T C_{DD}^{-1} C_{Dg} V_g 
\end{equation*}
Since \eqref{eq:free_energy} consists of terms which are at most linear in $V_g$, the boundary between neighbouring states $s$ and $r$ in state parameter space lies on a plane where the free energy of both states is equal and the plane fulfills
\begin{equation*}
    F(s,V_g)-F(r,V_g)=n_{sr}^T V_g + b_{sr} = 0\enspace.
\end{equation*}
Here, $n_{sr}=C_{Dg}^TC_{DD}^{-1}(s-r)$ is the normal of the plane and $b_{sr}= \lvert e\rvert/2( s^TC_{DD}^{-1} s - r^TC_{DD}^{-1}r)$ is its offset.
Since all boundaries are linear, the set of gate voltages $V_g$ that lead to a given state $s$ form a convex polytope created by the intersection of the planes:
\begin{align}
    P(s) &=\left\{ V_g \in \mathbb{R}^{n_g} \mid s^*(V_g)=s \right \}\notag\\
    &= \left\{ V_g \in \mathbb{R}^{n_g} \mid \forall r \in \mathbb{N}^{n_D}:\, F(s,V_g)-F(r, V_g) \leq 0 \right \}\notag\\
    &=\left\{ V_g \in \mathbb{R}^{n_g} \mid \forall r \in \mathbb{N}^{n_D}:\, n_{sr}^TV_g + b_{sr} \leq 0 \right \}\label{eq:CIMPolytope}
\end{align}
While the definition of \eqref{eq:CIMPolytope} runs over an infinite number of states, the number of facets which compose the boundary are finite. We will refer to states $r$ which share a facet with $s$ in $P(s)$ as neighbouring states. Thus, the polytope carries information about which state transitions to neighbouring states exist and can be used to control the device.

\section{Polytope Estimation}\label{sec:alg}
In this section, we describe our main contribution, an algorithm that can be used to find the facets of a convex polytope from measurements.
First, we introduce a meta algorithm that describes our active-learning procedure. It relies on an algorithm to estimate a convex polytope from measurements, which we describe afterwards.
\subsection{A Meta Algorithm For Finding Polytopes}

We employ a meta algorithm to iteratively improve an estimate $\hat{P}$ of a convex polytope $P$ based on line searches. For this, we use an algorithm $\mathcal{A}$ which computes $\hat{P}=\mathcal{A}(X)$, where $X$ is a training dataset of previously collected measurements. Our goal is to use active learning to propose new line search directions in order to add new points to $X$.
The meta algorithm is based on the following lemma:
\begin{lemma}\label{theo:equal}
Let $P$ and $\hat{P}$ be two convex polytopes in $\mathbb{R}^d$. Let $c_k, k=1,\dots, \hat{N}_c$ be the vertices of $\hat{P}$ and $f_l$, $l=1,\dots, \hat{N}$ its facets with normals $\hat{A}_l$. If
\begin{enumerate}
    \item $c_k \in \partial P,\;\forall k=1,\dots, \hat{N}_c$
    \item for each facet $f_l$, there exists a point $x_l$ inside of $f_l$ such that $x_l \in \partial P$,
\end{enumerate}
then $P=\hat{P}$.
\end{lemma}
\begin{proof}
From the first point, we directly obtain $\hat{P} \subseteq P$ due to the definition of convexity.
To show $P \subseteq \hat{P}$, pick any $l \in\{1,\dots, \hat{N}\}$. Since $x_l \in \partial P$, there must exist a direction $p_l \in \mathbb{R}^d$ such that all $x \in P$ fulfill the constraint $p_l^T x - p_l^Tx_l \leq 0$. Since $x_l$ lies inside a facet of $\hat{P}$, we also have $\hat{A}_l^Tx - \hat{A}_l^Tx_l\leq 0$ for all $x \in \hat{P}$. Finally, as $\hat{P}\subseteq P$ and $x_l$ is inside of $f_l$, we have $p=\hat{A}_l$. As this holds for all facets of $\hat{P}$, we obtain $P \subseteq \{x \in \mathbb{R}^d \mid \hat{A}_l^Tx - \hat{A}_l^Tx_l\leq 0,\;\forall l=1,\dots, \hat{N} \}=\hat{P}$.
\end{proof}
This suggests a simple algorithm for iteratively improving an estimate of a convex polytope $\hat{P}$, by generating a set of query points $R$ on the boundary of $\hat{P}$ which includes the vertices of $\hat{P}$ as well as a point inside each facet. To obtain points inside the facets, we take all vertices which are on a given facet and compute their mean. Then, for each point $R_i$, we perform a line search from a point $o \in P \cap \hat{P}$ (for example the mean of points in the training set) in direction of $R_i-o$ and search for the intersection with $P$. All intersections (or a bracket containing the likely intersection interval) are added to $X$.
The algorithm terminates if all new points in $X$ lie sufficiently close to the boundary of $\hat{P}$, i.e., $\text{dist}(\partial \hat{P},x)< \epsilon_\text{end}, \forall x \in X$. Otherwise, we repeat using the updated training set.

\subsection{Large Margin Polytopes}
\label{sec:LMPoly}
The algorithm introduced above works well given an exact line search. In this case, we could choose $\mathcal{A}$ as a convex hull algorithm. However, this has two disadvantages:
\begin{enumerate}
 \item In the presence of noise in the line search, points within a facet will never lie on a hyperplane.
 \item Even with exact line searches, the number of facets found by the algorithm tends to increase over time, far exceeding the true number of facets. This is because line searches will almost surely not hit a true vertex of $P$, resulting in additional facets.
\end{enumerate}
Both problems can be understood by seeing the convex hull as a non-parametric model for estimating a convex polytope. Each sampled point increases the complexity of the model, which makes it prone to overfitting, and thus requires regularization. Instead of using a convex hull, we model this estimation problem via regularized large margin classification~\citep{boser1992training}. To this end, we require that the line search does not return an estimate of the point $x \in \partial P$, but instead a confidence interval $(x^-,x^+)$, where $x^- \in P$ and $x^+ \notin P$. For this, an existing line search can be adapted, for example by adding confidence bounds to its estimate. As a decision function to optimize, we use 
\begin{equation}
    f(x)=\max_{k\in \{1, \dots, \hat{N} \}}\hat{A}_k^Tx +\hat{b}_k\enspace,
\end{equation}
as with this choice it holds $\hat{P}=\{x\in \mathbb{R}^d \mid f(x)\leq 0\}= \{x \in \mathbb{R}^d \mid \hat{A}_k^Tx +\hat{b}_k \leq 0, k=1,\dots, \hat{N}\}$. $\hat{A}$ and $\hat{b}$ are the optimizable parameters and $\hat{N}$ must be chosen larger than the number of facets in $P$. With this model, the task is to find a function $f$ which separates the training set of points $x^-_j, j=1,\dots, \ell$ from $x^+_j, j=1,\dots, \ell$. Our approach is similar to the large margin approach in \cite{kantchelian2014large}. The soft-constrained large margin primal optimization problem can be stated as

\begin{align}
    \min_{\hat{A}, \hat{b}}&\; \Omega(\hat{A})
    +\frac C { \ell}\sum_{i=1}^\ell (\xi^+_i)^2 + \frac C { \ell}\sum_{i=1}^\ell (\xi^-_i)^2 \\
    \text{s.t.} &\; \max_j \hat{A}_j^T x^-_i +\hat{b}_j \leq - 1 + \xi^-_i, i=1,\dots, \ell \label{eq:neg_const} \\
    \wedge&\;  \max_j \hat{A}_j^T x^+_i +\hat{b}_j \geq 1 - \xi^+_i, i=1,\dots, \ell \label{eq:non_convex} \\
    \wedge&\;  \xi^-_i\geq 0 \wedge \xi^+_i \geq 0, i=1,\dots, \ell \enspace, \nonumber 
\end{align}
where $\Omega: \mathbb{R}^{\hat{N} \times d} \rightarrow \mathbb{R}$ is a regularizer with complexity parameter $C$ and $\xi^+_i$ and $\xi^-_i$ are slack variables. The constraint \eqref{eq:non_convex} turns the problem into a non-convex optimization problem, and to find the global optimum, we would need to infer correctly for each $\hat{A}_j$ those pairs of points that are separated by it (or, in the context of the application, the state of the device associated with $x_i^+$).

The problem can be relaxed and solved iteratively using a sequence of convex problems with solutions $(\hat{A}^{(t)}, \hat{b}^{(t)})$, $t=1,2,\dots$. This can be achieved via a convex relaxation  of \eqref{eq:non_convex} based on the fact that if we find any hyperplane $\hat{A}_{s_i},\hat{b}_{s_i}$ for which $\hat{A}_{s_i}^Tx^+_i +\hat{b}_{s_i} > 1 - \xi^+_i$,
we automatically fulfill constraint \eqref{eq:non_convex}. Thus, in iteration $t$ of the algorithm, we can take the previous estimate $\hat{A}^{(t-1)}, \hat{b}^{(t-1)}$ to find $s_i$ and solve the subproblem

\begin{align}
    \min_{\hat{A}, \hat{b}} &\; \Omega(\hat{A})
    +\frac C { \ell}\sum_{i=1}^\ell (\xi^+_i)^2 + \frac C { \ell}\sum_{i=1}^\ell (\xi^-_i)^2  \label{eq:conv_problem}\\
    \text{s.t.}&\; \max_j \hat{A}_j^Tx^-_i +\hat{b}_j \leq - 1 + \xi^-_i\nonumber\\
    \wedge&\; \hat{A}_{s_i}^Tx^+_i +\hat{b}_{s_i} \geq 1 - \xi^+_i, 
    s_i=\arg \max_s\left \{ \left(\hat{A}^{(t-1)}_s\right)^T x^+_i +\hat{b}^{(t-1)}_s\right \}\label{eq:non_convex_relaxed}\\
    \wedge &\; \xi^-_i \geq 0 \wedge \xi^+_i \geq 0, i=1,\dots, \ell \nonumber \enspace.
\end{align}
In the algorithm by \cite{kantchelian2014large}, $\Omega$ was chosen as the squared $L_2$-norm of the rows of $\hat{A}$, $\Omega(\hat{A})= \sum_{i=1}^{\hat{N}} \lVert \hat{A}_i \rVert^2$ and the problem was solved via stochastic gradient descent using a variant of the Pegasos \citep{shalev2011pegasos} solver, that 
re-estimated $s_i$ in \eqref{eq:non_convex_relaxed} during each stochastic gradient step. Additionally, the solver in \cite{kantchelian2014large} would assign points to boundaries that are not assigned enough points to prevent a sparse solution. 
In initial experiments, this solver did not lead to usable or even plausible results in our setting. This is likely due to the ill-conditioning introduced by creating pairs of samples which are very close to the boundary. The resulting ill-conditioning of up to $10^8$ necessitates second order optimization algorithms.

\begin{algorithm}[H]
\caption{Algorithm for approximate solution of the polytope estimation problem}
\label{alg:alg_outer_loop}
\begin{algorithmic}[1]
\STATE{\textbf{input} $X^+, X^-$, $n_{\text{repeat}} \in \mathbb{N}^+$, $C > 0$, $\sigma > 0$}
\STATE{\textbf{output} Estimated polytope parameters $\hat{A}^{\text{best}}, \hat{b}^{\text{best}}$}
\STATE{$\hat{A}^{\text{init}}, \hat{b}^{\text{init}}  \leftarrow$ 
\texttt{convex\_hull\_polytope}$(X)$}
\STATE{$\hat{b}^{\text{init}}\leftarrow \hat{b}^{\text{init}} - 1$}
\STATE{Solve \eqref{eq:conv_problem} starting from $\hat{A}^{\text{init}}, \hat{b}^{\text{init}}$ and prune to obtain $\hat{A}_{\text{sol}}^{(0)}, \hat{b}_{\text{sol}}^{(0)}$  with value $L^{(0)}$ }

\FOR{$i = 1,\dots,n_{\text{repeat}}$}
    \STATE{$ \hat{A}^{\text{noise}} \leftarrow  \hat{A}_{\text{sol}}^{(0)} + \mathcal{N}(0, \sigma^2)$}
    \STATE{$ \hat{b}^{\text{noise}} \leftarrow  \hat{b}_{\text{sol}}^{(0)} + \mathcal{N}(0, \sigma^2)$}
    \STATE{Solve \eqref{eq:conv_problem} starting from $\hat{A}^{\text{noise}}, \hat{b}^{\text{noise}}$ and prune to obtain $\hat{A}_{\text{sol}}^{(i)}, \hat{b}_{\text{sol}}^{(i)}$ with value $L^{(i)}$ }
\ENDFOR
\STATE{$i^*\leftarrow \arg \min_i L^{(i)}$}
\STATE{$\hat{A}^{\text{best}}, \hat{b}^{\text{best}} \leftarrow \hat{A}_{\text{sol}}^{(i^*)}, \hat{b}_{\text{sol}}^{(i^*)}$}

\end{algorithmic}
\end{algorithm}

We therefore included the following changes:
To minimize the number of facets of the polytope, we use a regularizer, which enforces setting whole rows $\hat{A}_i$ to zero. To this end, we propose to use the $L_{2,1}$~regularizer,
$L_{2,1}(\hat{A})=\sum_{i=1}^{\hat{N}} \lVert \hat{A}_i \rVert_2$.
This regularizer can be seen as a generalization of 1-norm~regularization to impose sparsity on row vectors and has among other applications been used successfully in the context of joint feature selection for multi task learning~\citep{liu2012multi}. With this regulariser, the problem \eqref{eq:conv_problem} becomes a second order cone program (SOCP) which we minimize in each iteration using a second order solver before re-computing $s_i$. We stop the optimization when the $s_i$ in two iterations are the same. To save computation time, we prune the solution of each subproblem by removing all hyperplanes with $\hat{A}_i = 0$ as they won't get assigned new points and thus they would remain zero due to the regularization term. 
Finally, we use a restarting strategy to find a good local optimum as described in Algorithm~\ref{alg:alg_outer_loop}. We first create an initial value $\hat{A}^{\text{init}}, \hat{b}^{\text{init}}$ by computing the convex hull of the points $x_j^-$. We subtract 1 from $\hat{b}^{\text{init}}$ so that the initial solution fulfills \eqref{eq:neg_const} with $\xi^- = 0$. Afterwards, we repeatedly solve problem \eqref{eq:conv_problem} to obtain a first local optimum $\hat{A}_{\text{sol}}^{(0)}, \hat{b}_{\text{sol}}^{(0)}$. In our experience, this solution still has too many facets. Thus, we repeatedly add noise to this solution, solve the problem and return the best solution found\footnote{We applied the same strategy to the original optimization problem of \cite{kantchelian2014large} as an ablation study, but numerical difficulties made it impossible to solve the problem reliably.}.

\section{Experiments}
\label{sec:experiments}
We implemented our algorithm by solving problem \eqref{eq:conv_problem} with cvxpy~\citep{diamond2016cvxpy, agrawal2018rewriting} using the ECOS solver for SOCP problems~\citep{Domahidi2013ecos}.
We compute halfspace intersections, convex hulls and their volumes using the Qhull library~\citep{barber1996quickhull}.
To save computation time, we include new point pairs $(x^-,x^+)$ into the training dataset only if it does not already contain any pair $(x^-_i,x^+_i)$ with $\lVert x_i^- - x^-\rVert < \epsilon_{\text{close}}$. The algorithm terminates when the distance between the sampled points $(x^-_i,x^+_i)$ to the boundary of the polytope is smaller than $\epsilon_{\text{end}}$.
To assess the performance of our algorithms, we compared our approach, where possible, to an algorithm that has perfect information regarding which point pair is cut by which hyperplane. For this, we additionally compute the state $s_i$ of $x^+_i$, pick $\hat{N}$ as the number of different observed states and solve \eqref{eq:conv_problem} using the obtained $s_i$. This way, the problem becomes convex and can be solved efficiently. We refer to this algorithm as our baseline.

To quantify the quality of the obtained convex polytopes, we define a matching error between the set of facets in the ground truth $f_l =\{x \in P \mid A_l^Tx + b_l = 0\} \in F$ and the facets in the estimated convex polytope $\hat{f}_k = \{x \in \hat{P} \mid \hat{A}_k^Tx + \hat{b}_k = 0\}\in \hat{F}$.
For this, we compute the minimum angle between the normal of the facet $f_l$ and the normals of all $\hat{f}_k$: We report an error if there is no $\hat{f}_k$ with angle smaller than 10 degrees:
\[
    E = \frac {1}{\lvert F \rvert} \sum_{l=1}^{\lvert F \rvert}  
    \mathbb{I} \left\{\min_{k} \arccos\left( \frac{A_l^T \hat{A}_k}{\lVert A_l \rVert \lVert \hat{A}_k \rVert}\right) \cdot \frac {180} {\pi} > 10 
\right\}
\]
This is a sufficient metric for estimating closeness of the facets: as the stopping condition of the algorithm already ensures that the estimated polytope must be very similar to the true polytope in terms of intersection over union, it is only the directions of the facet hyperplanes and their number that needs to be assessed here. To verify this claim, we also compute the intersection over union between the estimated and true ground truth polytopes:
\[
\text{IoU}(P, \hat{P}) = \frac{\text{Vol}(P \cap \hat{P})}{\text{Vol}(\hat{P}) + \text{Vol}(\hat{P}) - \text{Vol}(P \cap \hat{P})}
\]

To test our algorithms, we devise three experiments: Two simulated problems as well as
one application on a real device. As application, we consider the device shown in Figure~\ref{fig:device_experiment}(a),
consisting of a narrow silicon nanowire through which electrons can flow from source (S) to drain (D). Four gate electrodes (G$_{1,\dots4}$) are capacitively coupled to the nanowire, and under correct tuning of their voltages, quantum dots are formed that confine single electrons in small regions (tens of nanometers) of the nanowire.
One of the gate electrodes (G$_{4}$) is connected to a high frequency circuit and produces a sensor signal that responds very quickly to any rearrangements of electrons inside the nanowire. This sensor signal remains constant under voltage changes (line searches) of gate electrodes G$_{1}$, G$_{2}$, and G$_{3}$, as long as the state
does not change. 
Experimentally, this is accomplished by compensating $V_4$ for any voltage changes applied to G$_{1,2,3}$, such that the potential of the sensor dot is independent of the three control voltages $V_{1,2,3}$. This requires estimating the coefficients of a linear compensation function, which was done once before the algorithm is run.
As soon as the charge state changes, i.e., a state boundary is encountered, the sensor signal drops and generates one $(x^-_i,x^+_i)$ pair. The device in Figure~\ref{fig:device_experiment}(a) can create up to four quantum dots.
We applied our algorithm to a state with dot 1 and dot 2 each occupied by one electron, and use the two gate voltage parameters $V_1$ and $V_2$ to control this $[1,1]$~configuration. (Dot 3 is kept empty by keeping $V_3$ fixed at a sufficiently negative voltage). We repeat this experiment tuned such that an electron occupies the dot under G$_3$ as well and estimate the polytope using all three gate voltage parameters in the $[1,1,1]$~configuration.

For the simulated experiments, we implement a line search which returns point pairs such that $\lVert x_i^+-x_i^-\rVert < \delta$, where we vary $\delta$ across the experiments to measure the impact of the quality of the line search.
As our first simulated problem, we generate a set of polytopes via Voronoi regions. For this, we sample 30 points from a $d$-dimensional normal distribution in $\mathbb{R}^d$ with zero mean and diagonal covariance matrix $\Sigma$ with entries $\Sigma_{ii}= 2 \cdot 10^{\frac i d}$. For each set, we compute a Voronoi triangulation. We discard the sample if the origin is not inside a closed Voronoi region or this region extends outside the set $[-10,10]^d$. 
On this problem, we can use the Voronoi region around the origin as ground truth. This way, we obtain 100 convex polytopes with between 6 and 12 facets for $d=3$ and $d=4$. For the baseline algorithm, we use the index of the closest point as state of the Voronoi region.

As our second problem, we simulate a quantum dot array with 3 or 4 qubit dots using the constant interaction model. For simplicity, we omit an additional sensor dot that would normally sense the charge transitions within these triple- and quadruple-dot devices. Specifically, we generate 100 different device simulations by choosing one set of realistic device capacitances and adding noise to simulate variations in device manufacturing. To be able to use a similar tuning as in the Voronoi experiment, we re-scale the parameter space by a factor of 100 so that the polytopes cover roughly the $[-10,10]^d$ volume. 
As the ground truth polytope for each simulated device we compute that region in control-voltage space that has a ground state with a [1,1,1] or [1,1,1,1] charge configuration, i.e. exactly one electron per dot. These calculations yield 14 facets for all simulated triple-dot devices, and 30 facets for all simulated quadruple-dot devices. Errors made by the algorithm in finding the correct facets are discussed below.    

As parameters, we chose in all experiments $\epsilon_{\text{end}} = 1.5\delta$, $\epsilon_{\text{close}}=\delta$ and $n_{\text{repeat}} = 10$.
In the simulated experiments, we vary $\delta \in \{0.01, 0.05, 0.1, 0.2\}$, choose $\sigma = 0.001/ \delta$ and $C=75 / \delta$. For the baseline algorithm, we use the same parameters except $C=750 / \delta$. We obtain the initial dataset $X$ by performing 100 line searches in random directions starting from a known point inside $P$ (e.g., the origin for the Voronoi region based dataset).
For the real device, we implemented a line search in hardware using a DAC voltage generator. 
Instead of using the single point $x_i$ returned by the hardware, we add a confidence interval $x_i^{\pm} = x_i \pm \delta / 2$ with $\delta=0.001$, which takes into account the experimental measurement uncertainty. This choice of $\delta$ is equivalent to the $\delta=0.1$ setting in the simulated device on the re-scaled parameters. Further, we pick $\sigma = 0.05/ \delta$ and $C=10 / \delta$. The initial dataset was obtained via line searches using 10 random directions.

\section{Results}\label{sec:results}
\begin{figure}
	\centering
	\begin{subfigure}{0.45\textwidth}
		\centering
	\includegraphics{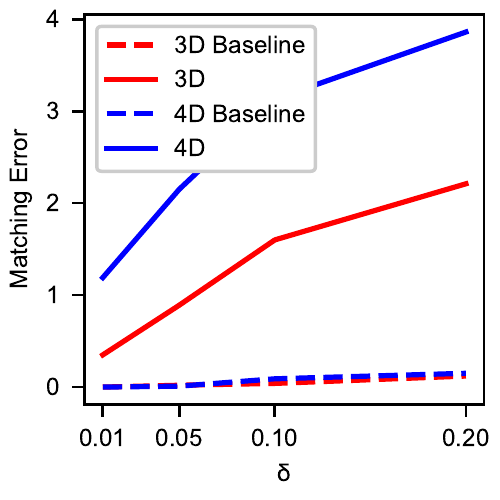}
		\caption{Voronoi}
	\end{subfigure}%
	\begin{subfigure}{0.45\textwidth}
		\centering
		\includegraphics{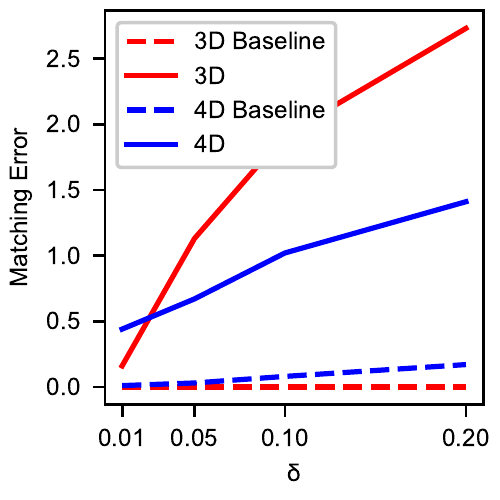}
		\caption{Simulator}
	\end{subfigure}%
	\caption{Average number of matching errors made by the algorithm as function of line search precision (see text). We show results on the 3D and 4D problems for our algorithm and the idealized baseline. Left: Voronoi region problem. Right: Problems drawn from the quantum device simulator.}
	\label{fig:error_alg}
\end{figure}
In our experiments, we first investigated how close the learned polytope is to the true polytope on the simulated devices, depending on the line search accuracy $\delta=\lVert x^+ - x^- \rVert$. For this, we computed the intersection over union of true and estimated polytopes and found an expected value between 0.9996 ($\delta=0.01$) and 0.995 ($\delta=0.02$), which are both very close to the optimum of 1. We did neither observe any relevant differences to the baseline algorithm, nor large differences between 3 and 4 dimensions. We then measured how many of the linear decision functions in the ground truth polytope were not approximated well in the estimated polytope. The recorded number of matching errors on the two simulated datasets as a function of the precision of the line search, $\delta$, can be seen in Figure~\ref{fig:error_alg}.
For both datasets, the error increased with $\delta$, where for the simulated devices, we saw a far steeper increase than for the Voronoi regions. For small values of $\delta$, we observed less than 0.1 matching errors, which increased to 4.5 at $\delta=0.2$ in 4D. Here, the baseline algorithm performed better.

\begin{figure}
	\begin{subfigure}{0.45\textwidth}
		\centering
	\includegraphics{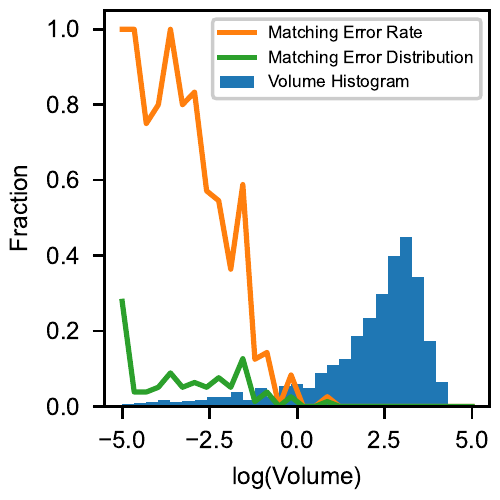}
		\caption{Voronoi,$\delta=0.2$}
	\end{subfigure}%
	\begin{subfigure}{0.45\textwidth}
		\centering
	\includegraphics{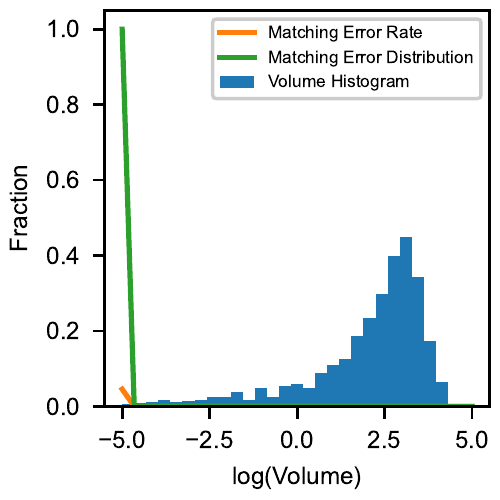}
		\caption{Voronoi,$\delta=0.01$}
	\end{subfigure}\\
	\begin{subfigure}{0.45\textwidth}
		\centering
		\includegraphics{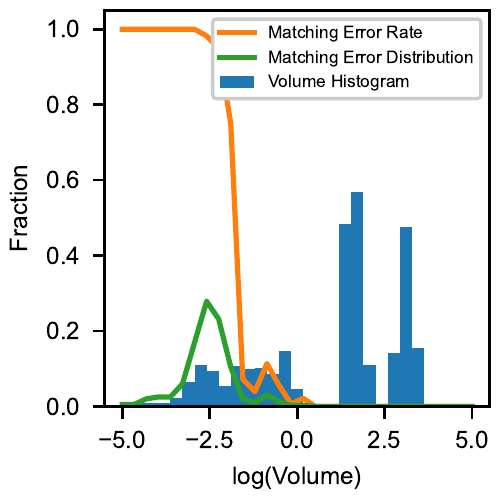}
		\caption{Simulator,$\delta=0.2$}
	\end{subfigure}%
	\begin{subfigure}{0.45\textwidth}
		\centering
		\includegraphics{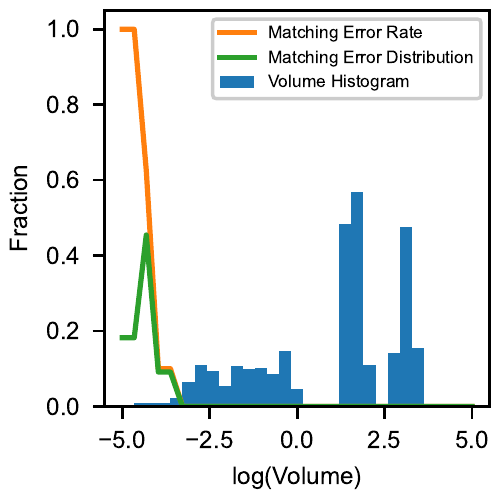}
		\caption{Simulator,$\delta=0.01$}
	\end{subfigure}%
	
	\caption{Histograms of the volumes of the facets of the convex polytopes in 4D, the relative distribution of matching errors and the error rate at a given volume and for different $\delta$. Top: Voronoi problem. Bottom: Simulated quantum devices.\label{fig:error_hist}}
\end{figure}
\begin{figure}
	\label{fig:dist_volume}
	\begin{subfigure}{0.45\textwidth}
		\centering
		\includegraphics{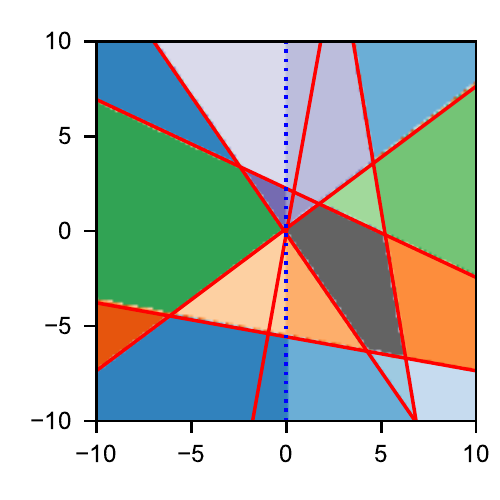}\vspace{-1.5em}
		\caption{Imprecise Facet}
		\label{fig:error_vis_4D_1}
	\end{subfigure}%
	\begin{subfigure}{0.45\textwidth}
		\centering
		\includegraphics{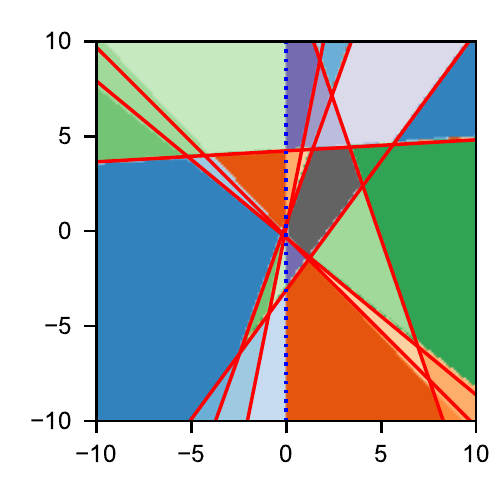}\vspace{-1.5em}
		\caption{Missing facet}
		\label{fig:error_vis_4D_2}
	\end{subfigure}%
	\caption{Examples of typical errors of the algorithm with $\delta=0.2$ on the simulated device. Shown are two different 2D cuts through the same 4D volume. Cuts are chosen such that a wrongly estimated facet appears as a vertical line through the origin (blue dashed line). Red lines represent the constraints of \emph{estimated} facets that intersect with the image plane. Areas are colored based on a partitioning of the plane by the ground truth constraints (not device states). The area belonging to the polytope is depicted in dark grey. An estimated constraint is correct if it aligns with a boundary of these areas. Units are displayed relative to the chosen origin.
	 Left: Error in estimating the correct angle of a small facet. Right: Missing edge due to small facet size.}
	\label{fig:error_vis_4D}
\end{figure}
To investigate the nature of the errors, we computed the hyper volume of the state transitions as a measure of size of the facet in the polytope. The histograms of volumes as well as relative distributions of errors are shown in Figure~\ref{fig:error_hist}. As can be seen from the figure, most errors appeared on facets with small volume. Further, we observed that facets with small volumes were estimated more reliably as the line search accuracy improved. To give a more qualitative impression of the errors observed, we visualized two examples of an error by creating random cuts through the 4D parameter space of the simulated device in such a way that an unmatched facet in the ground truth is visible within the cut. See Figure~\ref{fig:error_vis_4D} for common examples, extracted from a problem instance with 4 errors and obtained with $\delta=0.2$.

For the real device, the obtained model of our algorithm is shown in Figure~\ref{fig:device_experiment} in red color. In 2D, human experts verified the correctness of the state transitions and annotated the resulting states, based on a dense 2D raster scan shown in the background of Fig.~\ref{fig:device_experiment}(b). We can see that the algorithm found two very small facets which are invisible on the raster scan. In 3D no ground truth was available, but qualitatively the 10 facets found by the algorithm (six large ones and two pairs of thin slabs in Fig.~\ref{fig:device_experiment}(c)) are in line with how the experts understand this device: The 6 large transitions correspond to adding or removing one electron from each qubit dot, whereas the two pairs of smaller facets correspond to moving an electron between dot 1 and dot 2, or between dot 2 and dot 3. We further verified in simulations of the constant interaction model, including a simulation of the sensor dot, that our polytope is consistent with specific choices of the device parameters.

\begin{figure}
    \centering
    
	\begin{subfigure}{0.45\textwidth}
		\centering
		\includegraphics{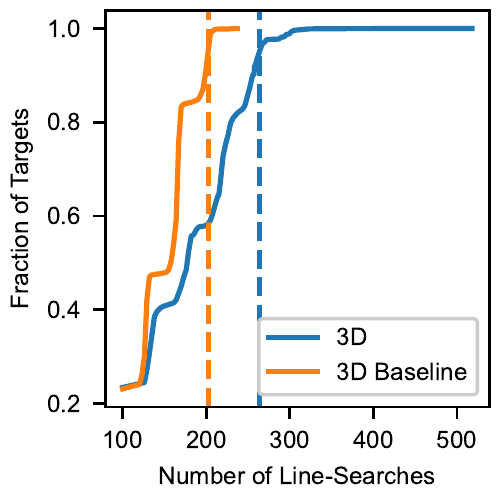}
		\caption{$d=3$, $\delta=0.01$}
	\end{subfigure}%
	\begin{subfigure}{0.45\textwidth}
		\centering
		\includegraphics{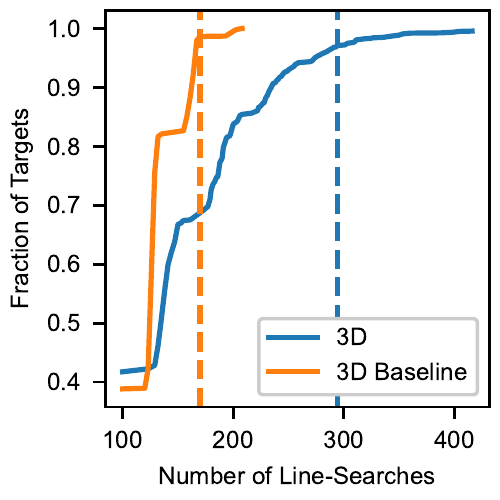}
		\caption{$d=3$, $\delta=0.2$}
	\end{subfigure}
	\begin{subfigure}{0.45\textwidth}
		\centering
		\includegraphics{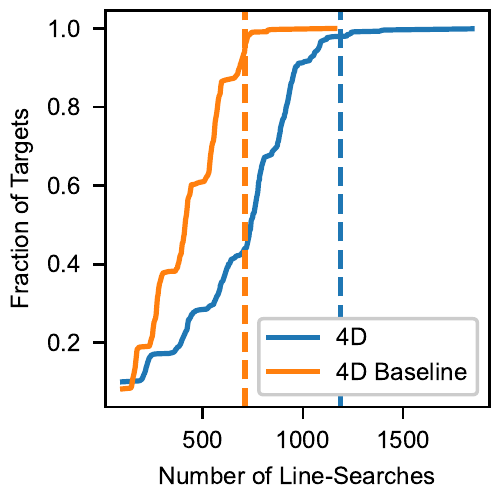}
		\caption{$d=4$, $\delta=0.01$}
	\end{subfigure}%
	\begin{subfigure}{0.45\textwidth}
		\centering
		\includegraphics{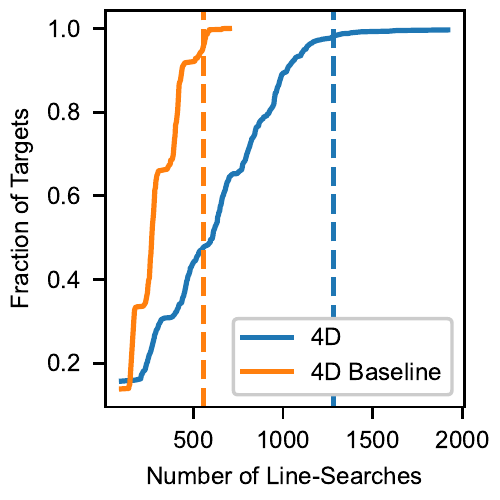}
		\caption{$d=4$, $\delta=0.2$}
	\end{subfigure}
	\caption{Estimated cumulative distribution function (ECDF) depicting the fraction of reached precision targets over number of line searches for the simulated device. For each run there are 20 reachable targets, geometrically spaced between $\epsilon_{\text{end}}=1.5 \cdot \delta$ and $10$. Let $X_i$ be the training set used to train model $\hat{P}_i$ in the $i$-th iteration of Algorithm~\ref{alg:alg_outer_loop}. The $j$th target $t_j$ is considered reached if we find an $i$ such, that $\max_{x \in X_{i+1}} \text{dist}(\partial \hat{P}_i,x) < t_j$. The $|X_i|$ for the smallest $i$ fulfilling this is used as the number of line searches required. Vertical dashed lines represent the point where 75\% of runs are terminated.}
	\label{fig:ecdf_searches}
\end{figure}
Finally, we investigated the run time and convergence of the algorithm. We measured the convergence speed of the algorithm in terms of the maximum deviation between transitions found by the line search and the model prediction. The results are shown in Figure~\ref{fig:ecdf_searches}. The number of line searches conducted on the simulated device in 4D was approximately 1500, while the baseline needed 500 line searches, with no clear tendencies depending on $\delta$. The run time of the algorithm for the real device (including line searches) was less than a minute for two dots and 30 minutes for three dots (15 minutes measurement time for 180 line searches). Computing an instance of the 4D problem on the simulated device takes approximately one hour on a single core CPU.

\section{Discussion and Conclusion}\label{sec:conc}
Our results show that it is possible to reliably estimate the polytope associated with state transitions out of a specific charge state in a quantum dot array. Overall, it appears that the algorithm finds a good estimate of the true polytope volumes and facet shapes within the control parameter space formed by the array's gate voltages. While the underlying optimization problem is non-convex, we did not find signs of bad local optima. This is likely because we verify the estimated polytope using new measurements which are designed to find potential modeling errors. Thus, a local optimum is likely caught and results in an new training iteration. This hypothesis is strengthened by the fact that we observe that with higher precision of the line search the distribution of errors is shifted towards smaller facets, showing that we are better at disproving local optima.

An example of the superiority of our approach to rastering can be seen in
Figure~\ref{fig:device_experiment}(b). Here, we can see that the algorithm managed to find two very small facets (labeled [0,2] and [2,0]) which are invisible on the raster scan. Operationally, these are two very important facets in the device as they amount to transitions of electrons \emph{between} quantum dots, while the other transitions are electrons entering or leaving the array. For the purposes of quantum computation, such inter-dot tunneling processes effectively turn on wave-function overlap between two electrons, which is a key resource to manipulate their spin states via Heisenberg exchange coupling~\citep{Cha21}. The same holds for the 3D polytope estimated in Figure~\ref{fig:device_experiment}(c), where small facets corresponding to inter-dot transitions have been found. 

Compared to random sampling, our sampling approach using active learning seems to be superior. For random sampling, the number of samples requires depends largely on the relative size of a facet compared to the overall surface area \cite{weber2022theoretical}. Thus, for the small size of facets on our polytope, an average of 1500 samples is significantly less than expected via random sampling.

Still, our algorithm is not perfect. Comparing the performance of our algorithm to the idealized baseline algorithm indicates that the majority of errors are introduced by the difficulties of solving the NP-hard polytope estimation problem, while our active learning scheme seems to works very well. Thus, our work will likely profit from continued development of improved and faster estimation algorithms.

We have successfully applied our algorithm to a real 2x2 quantum dot array, constituting the first automatic discovery of state transitions in the literature. As the underlying estimation problems are NP-hard and we observed a steep increase of run-time already going from 3D to 4D, we foresee practical limits of our algorithm for controlling large quantum dot arrays. However, in practice, one may not need to find all facets in order to implement desired qubit functionalities. For example, universal quantum computers can be constructed entirely from single- and two-qubit operations, meaning that many facets associated with a multi-dot state may be operationally irrelevant (such as facets corresponding to multi-electron co-tunneling transitions, which are likely very small facets that are hard to find). Alternatively, it might be possible to partition the device in blocks of smaller arrays, for which polytopes can be estimated independently. For example, for a linear array of many quantum dots, it might be possible to estimate all relevant local transitions of a quantum dot using its two-neighbourhood in the dot connection graph, a 5-dimensional subspace, which we have shown is already within our estimation capabilities. Another possible solution is to use more and better sensor signals to identify the state transitions.

A potential shortcoming is our assumption of linear state transitions, which form a convex polytope. This assumption is not true for devices that exhibit device instabilities. However, stable, non-hysteretic materials are essential prerequisites for quantum processors, and recent material improvements already led to very stable materials. Thus, we expect that future devices will fulfill this assumption to a much higher degree than devices that are available today.

Aside from requiring an initial point inside a target state, the algorithm runs unsupervised and automatically proposes new measurement points to refine the estimate of the state transitions. Thus, we foresee our method to have major impact in tuning quantum dot devices where human tuning capability is the limiting factor, even on the small devices operated today.

\section*{Acknowledgements}
This work received funding from the European Union under grant agreement No. 688539 and 951852. F.K. and A.C. acknowledge support from the Independent Research Fund Denmark.
Contributions: O.K. developed the algorithm and designed, built and executed the experiments on the simulated device. A.C., T.R. and B.B. developed the simulator and conducted the experiments on the real device. F.K. led the project. All authors contributed to writing the article.

\bibliography{references}

\end{document}